\documentclass{article}

\usepackage{url}
\usepackage{mathrsfs}
\usepackage{latexsym}
\usepackage{amsmath}
\usepackage{amsfonts}
\usepackage{amssymb}
\usepackage{amsthm}
\usepackage{geometry}

\usepackage{hhline}
\usepackage{times}
\usepackage{graphicx} %
\usepackage{algorithm}
\usepackage{algorithmicx,algpseudocode}
\usepackage{tikz}
\usetikzlibrary{positioning}
\usepackage{epstopdf}

\usepackage[round]{natbib}

\usepackage{color}

\newcommand{\marg}{\operatorname{marg}}
\newcommand{\compl}{\mathsf{complexity}}
\newcommand{\empir}{\mathsf{empirical}}
\newcommand{\gener}{\mathsf{gen}\mydash\mathsf{err}}
\newcommand{\err}{R}
\newcommand{\serr}{\widehat{R}_n}

\newcommand{\wh}{\widehat}

\newcommand{\mc}[1]{\mathcal{#1}}

\newcommand{\radem}{\mc{R}}
\newcommand{\dist}{\rho} %

\mathchardef\mhyphen="2D

\newcommand{\loss}{\mathcal{L}}
\newcommand{\norm}[1]{\| #1 \|}

\newcommand{\g}{\gamma}

\newcommand{\ddim}{\operatorname{ddim}}
\renewcommand{\dim}{d}
\newcommand{\Lip}[1]{\nrm{#1}_{\textrm{{\tiny \textup{Lip}}}}}
\newcommand{\sign}{\operatorname{sign}}

\newcommand{\argmax}{\mathop{\mathrm{argmax}}}
\newcommand{\inv}{^{-1}} %

\newcommand{\knn}{_{k\mydash\textrm{{\tiny \textup{NN}}}}}

\newcommand{\mydash}{\text{-}}

\newcommand{\X}{\calX}
\newcommand{\Y}{\mathcal{Y}}

\newcommand{\F}{\mathcal{F}}

\usepackage{dsfont}
\newcommand{\chr}{\mathds{1}}
\newcommand{\pred}[1]{\chr_{\left\{ #1 \right\}}}

\newcommand{\E}{\mathbb{E}}

\newcommand{\diam}{\operatorname{diam}}

\renewcommand{\P}{\mathbb{P}}

\newcommand{\ben}{\begin{enumerate}}
\newcommand{\een}{\end{enumerate}}
\newcommand{\bit}{\begin{itemize}}
\newcommand{\eit}{\end{itemize}}

\def\clap#1{\hbox to 0pt{\hss#1\hss}}

\newcommand{\nrm}[1]{\left\Vert #1 \right\Vert}

\newcommand{\calX}{\mathcal{X}}

\newcommand{\R}{\mathbb{R}}
\newcommand{\N}{\mathbb{N}}

\newcommand{\beq}{\begin{eqnarray*}}
\newcommand{\eeq}{\end{eqnarray*}}
\newcommand{\beqn}{\begin{eqnarray}}
\newcommand{\eeqn}{\end{eqnarray}}
\newcommand{\paren}[1]{\left( #1 \right)}

\newcommand{\tlprn}[1]{\left\{ #1 \right\}}
\newcommand{\set}[1]{\tlprn{#1}}
\newcommand{\abs}[1]{\left| #1 \right|}
\newcommand{\smabs}[1]{| #1 |}

\newcommand{\gn}{\, | \,}
\newcommand{\ds}{\displaystyle}

\newcommand{\hide}[1]{}
\newcommand{\oo}[1]{\frac{1}{#1}}
\def\eps{\varepsilon}

\newtheorem{theorem}{Theorem}

\newtheorem{lemma}[theorem]{Lemma}

\newcommand{\bepf}{\begin{proof}}
\newcommand{\enpf}{\end{proof}}

\def\longto{\mathop{\longrightarrow}\limits}
\newcommand{\ninf}{\longto_{n\to\infty}}

\newcommand{\ninfas}{\longto_{n\to\infty}^{\mbox{a.s.}}}

\newcommand{\bay}{^*}

\newcommand{\risk}{R}
\newcommand{\emprisk}{\wh{\risk}}

\newcommand{\srm}{{\textrm{{\tiny \textup{PEN}}}}}

\newcommand{\sloss}{\Phi}

\newcommand{\srisk}{\mc{L}}
\newcommand{\empsrisk}{\wh{\mc{L}}}
\newcommand{\nd}{n_{\dim}}
\newcommand{\gap}{\xi}

\newcommand{\cl}{\operatorname{cl}}

\newcommand{\pstar}{\mbox{(*)}}
\newcommand{\psstar}{\mbox{(**)}}
\newcommand{\pssstar}{\mbox{(***)}}

\newcommand{\cc}{\!} 

\title{A Bayes consistent 1-NN classifier}
\author{
Aryeh Kontorovich and Roi Weiss \\
Computer Science Department \\
Ben Gurion University\\
Beer Sheva, Israel \\
\texttt{\{karyeh,roiwei\}@cs.bgu.ac.il}
}

\begin{document}
\maketitle

\begin{abstract}
We show that a simple modification 
of the $1$-nearest neighbor classifier yields a 
strongly
Bayes consistent learner.
Prior to this work,
the only 
strongly
Bayes consistent proximity-based method
was the
$k$-nearest neighbor classifier,
for $k$ growing appropriately with sample size.
We will argue that a margin-regularized $1$-NN
enjoys considerable statistical 
and algorithmic advantages
over the $k$-NN classifier.
These include
user-friendly finite-sample error bounds,
as well as
time- and memory-efficient
learning and test-point evaluation algorithms with
a principled speed-accuracy tradeoff.
Encouraging empirical results are reported.
\end{abstract}

\section{Introduction}
The nearest neighbor (NN) classifier, introduced by 
Fix and Hodges in 1951, continues to be a popular learning algorithm 
among practitioners.
Despite the numerous sophisticated techniques developed in recent years,
this deceptively simple method continues to
``yield[] competitive results''
\citep{DBLP:journals/jmlr/WeinbergerS09}
and inspire papers 
in
``defense of nearest-neighbor based [\ldots] classification''
\citep{DBLP:conf/cvpr/BoimanSI08}.

In the sixty years since the introduction of the nearest neighbor paradigm,
a large amount of theory has been developed for analyzing this 
surprisingly effective classification method.
The first such analysis is due to \citet{CoverHart67}, 
who showed that as the sample size grows,
the $1$-NN classifier
almost surely approaches
an error rate $R\in[R^*,2R^*(1-R^*)]$, where $R^*$
is the Bayes-optimal risk. Although the $1$-NN classifier is not 
in general
Bayes consistent,
taking a majority vote among the $k$ nearest neighbors 
does guarantee strong Bayes consistency,
provided that $k$ increases appropriately in sample size
\citep{stone1977,devroye1985nonparametric,zhao1985exponential}.

The $k$-NN classifier 
in some sense 
addresses the Bayes consistency problem,
but
presents issues of its own. 
A naive implementation involves storing the entire sample,
over which a linear-time search is performed when answering queries on test points.
For large samples sizes, 
this approach is prohibitively expensive 
in terms of storage memory and computational runtime.
To mitigate the memory concern,
various {\em condensing} heuristics have been proposed 
\citep{hart1968condensed,gates72,ritter75,wilson00,gottlieb2018near}
--- of which only the one in \citet{gottlieb2018near} 
comes with any rigorous compression guarantees,
and only for 
$k=1$;
moreover, it is shown therein that the condensing problem is ill-posed for $k>1$.
Query evaluation on test points
may be significantly sped up via an {\em approximate} nearest neighbor search
\citep{KL04,beygelzimer2006cover,andoni2006near,DBLP:conf/colt/GottliebKK10}.
The price one pays for the fast approximate search is a degraded classification accuracy,
and 
of the 
works cited,
only \citet{DBLP:conf/colt/GottliebKK10}
quantifies 
this tradeoff
---
and again,
only for $1$-NN. 

On the statistical front, one desires a classifier that provides an easily computable {\em usable}
finite-sample generalization bound --- one that the learner can evaluate based only on the observed
sample so as to obtain a high-confidence error estimate. 
As we argue below, existing $k$-NN bounds
fall short of this desideratum, and the few known usable bounds given in
\citet{DBLP:journals/jmlr/LuxburgB04,DBLP:conf/colt/GottliebKK10,gottlieb2018near} are all for
$k=1$.

Motivated by the computational and statistical advantages that $1$-NN seems to enjoy over $k$-NN,
this paper presents a 
strongly
Bayes consistent $1$-NN classifier.

\paragraph{Main results.}
Our results build on the work of \citet{DBLP:conf/colt/GottliebKK10} and, more recently, \citet{gottlieb2018near}.
Suppose we are given an iid training sample $S$ consisting
of $n$ labeled points $(X_i,Y_i)$, 
with $X_i$
residing in some metric space $\X$ and $Y_i\in\set{-1,1}$.
For $\eps,\g>0$,
let us say that $S$ is
$(\eps,\g)$-{\em separable}
if there is a sub-sample
$\tilde S\subset S$
such that
\bit
\item[(i)]
the $1$-NN classifier
induced by $\tilde S$ mislabels at most $\eps n$ points in $S$
and
\item[(ii)]
every pair of opposite-labeled points in $\tilde S$
is at least $\g$ apart in distance.
\eit
Obviously, a given
sample $S$ 
cannot be 
$(\eps,\g)$-{separable}
for 
$\eps$ 
arbitrarily small
and $\g$
arbitrarily large.
Every $\g>0$ determines some minimum feasible $\eps^*=\eps^*(\g)$
and a corresponding 
$\eps^*$-consistent, $\g$-separable sub-sample $S^*(\g)\subset S$.

Margin-based generalization bounds were presented in
\citet{DBLP:conf/colt/GottliebKK10,gottlieb2018near}, with $\eps$ corresponding to empirical error
and $\g$ to the {\em margin}. Schematically, these bounds are of the form
\beqn
\label{eq:schem}
\gener_n(\eps,\g) \le \empir_n(\eps,\g) + \compl_n(\g),
\eeqn
where 
$\gener$ is the generalization error
of the $1$-NN classifier induced by an
$\eps$-consistent, $\g$-separable $\tilde S\subset S$,
and 
the two terms on the right-hand side correspond roughly to sample error and hypothesis complexity.
The approach proposed in 
\citet{DBLP:conf/colt/GottliebKK10,gottlieb2018near}
suggests computing $\eps^*(\g)$ for each $\g>0$ and minimizing the right-hand side
of (\ref{eq:schem}) over $\g$ to obtain $\g_n^*$. 
Indeed, the chief technical contribution of those works
consisted of providing efficient algorithms for computing
$\eps^*(\g)$, $S^*(\g)$, and $\g_n^*$.
In contrast, the present paper deals with the statistical aspects of this procedure.
Our main contribution is Theorem~\ref{thm:main}, which shows that the 
$1$-NN classifier induced by $S^*(\g_n^*)$ is strongly Bayes consistent.
Denoting this classifier by $h_n$,
our main result is formally stated as follows:
\beq
\P\paren{ h_n(X)\neq Y \gn (X_1,Y_1),\ldots, (X_n,Y_n)}
&\ninfas& R^*,
\eeq
where
$$
R^* = \inf_{h:\X\to\set{-1,1}}\P(h(X)\neq Y)
$$
is the Bayes-optimal error.
This is the first consistency result (strong or otherwise) for an
algorithmically efficient
$1$-NN classifier.

\paragraph{Related work.}
Following the pioneering work of 
\citet{CoverHart67},
it was shown by
\citet{devroye1985nonparametric,zhao1985exponential}
that the $k$-NN classifier
is strongly Bayes consistent.
A representative result for 
the Euclidean space
$\X=\R^d$ states that
if $k\to\infty$ and $k/n\to0$, then for all $\eps>0$ and 
$n>n_0(\eps,k)$,
\beqn
\label{eq:knn-gyorfi}
\P(\err(h\knn)>R^*+\eps) \le 2\exp\paren{-\frac{n\eps^2}{5184 \kappa_d^2}},
\eeqn
where
$\kappa_d<\paren{1+2/\sqrt{2-\sqrt{3}}}^d$
is the minimum number of 
origin-centered cones of angle $\pi/6$ that cover $\R^d$
(this result, among many others, is proved in \citet{MR1383093}).
Given the inherently Euclidean nature of $\kappa_d$, 
(\ref{eq:knn-gyorfi})
does not seem to
readily extend
to more general metric spaces.
It was (essentially) shown in
\citet{shwartz2014understanding} that
\beqn
\label{eq:knn-shai}
\E[\err(h\knn)] \le \paren{1+\sqrt{8/k}}R^* + \paren{6L+k}n^{-1/(\dim+1)}
\eeqn
for metric spaces $\X$ with unit diameter and doubling dimension $\dim$ (defined below),
where $L$ is the Lipschitz constant of 
$\eta:\X\to[0,1]$ defined by
$\eta(x)=\P(Y=1\gn X=x)$.
Recently, some of the classic results on $k$-NN risk decay rates were refined
by
\citet{chaudhuri2014rates}
in an analysis that captures the interplay between the metric and the sampling distribution.

Although (\ref{eq:knn-gyorfi},\ref{eq:knn-shai}) are both finite-sample bounds,
they do not enable
a practitioner to compute a numerical generalization error estimate for a given
training sample. Both are stated in terms of the unknown Bayes-optimal rate $R^*$,
and (\ref{eq:knn-shai}) additionally depends on $L$, a property of the unknown distribution.
In particular,
(\ref{eq:knn-gyorfi}) and (\ref{eq:knn-shai}) do not allow for a data-dependent selection
of $k$, which must be tuned via cross-validation.
The asymptotic expansions in \citet{snapp1998asymptotic,psaltis1994finite} 
likewise do not provide a computable
finite-sample bound.

An entire chapter in
\citet{MR1383093} is devoted to condensed and edited NN rules.
In the terminology of this paper, this amounts to extracting a sub-sample
$\tilde S$
and predicting via the $1$-NN classifier induced by that
$\tilde S$. 
Assuming a certain sample compression rate and an oracle for choosing an optimal
fixed-size 
$\tilde S$,
this scheme is shown to be weakly Bayes consistent.
The generalizing power of sample compression was independently 
discovered by \citet{warmuth86}, and later elaborated upon by 
\citet{graepel2005pac}.
In the context of NN classification,
\citet{MR1383093} list various condensing heuristics (which have no known performance guarantees)
and also leaves open the algorithmic question how to minimize the
empirical loss over all subsets of a given size.

The first substantial departure from the $k$-NN paradigm was proposed by
\citet{DBLP:journals/jmlr/LuxburgB04}, with the 
straightforward but far-reaching observation
that the $1$-NN classifier is, in some sense, 
equivalent to 
interpreting
the labeled sample $\set{(X_i,Y_i):i\in[n]}$
as $n$ evaluations of a real-valued target function $f$,
computing its
Lipschitz extension $f^*$ from the sample points to all of $\X$,
and then classifying test points by $\sign(f^*(\cdot))$.
Following up,
\citet{DBLP:conf/colt/GottliebKK10}
obtained bounds on the fat-shattering dimension of Lipschitz functions in doubling spaces
and gave 
margin-based risk bounds decaying
as $\tilde O(n^{-1/2})$ as opposed to $n^{-1/\dim}$.
More recently, the existence of a margin was leveraged to give nearly 
optimal sample compression bounds, with corresponding
generalization guarantees~\citep{gottlieb2018near}.


\section{Preliminaries}
\label{sec:prelim}
\paragraph{Metric spaces.} 
Throughout this paper, our instance space $\X$ will be endowed with a bounded metric $\dist$,
which we will normalize to have unit diameter\footnote{
This assumption is not really restrictive, as any finite sample will be contained in some ball.
The situation is analogous to margin-based analysis of Euclidean hyperplanes,
where the quantity of interest is the ratio between data diameter and 
geometric
margin.
}:
\[
\diam(\X):=\sup_{x,x'\in\X}\dist(x,x')=1.
\]
A function $f: \X \to \R$ 
is said to be $L$-Lipschitz if
$
\abs{f(x)-f(x')} \leq L \dist(x,x')
$
for all $x,x'\in\X$. 
The Lipschitz  constant of $f$, denoted $\Lip{f}$, 
is the smallest $L$ for which $f$ is $L$-Lipschitz.
The collection of all $L$-Lipschitz $f:\X\to[-1,1]$ will be denoted by
$\F_L$.
The distance between two sets $A,B\subset\X$ is defined by 
$\dist(A,B)=\inf_{x\in A,x'\in B}\dist(x,x')$.

For a metric space $(\X, \dist)$, let $\lambda$ be the smallest value such that every ball in $\X$ can be
covered by $\lambda$ balls of half the radius. 
The {\em doubling dimension} of $\X$ is $\ddim(\X):=\log_2 \lambda$. A metric is
{\em doubling} when its doubling dimension is finite.
We will denote $\dim:=\ddim(\X)<\infty$.

\paragraph{Learning model.}
We work in the standard {\em agnostic} learning model 
\citep{mohri-book2012,shwartz2014understanding},
whereby the learner receives a sample $S$ consisting of $n$ labeled examples
$(X_i,Y_i)$, drawn iid from an unknown distribution 
over $\X\times\set{-1,1}$. 
All subsequent probabilities and expectations will be with respect to 
this distribution.
Based on the training sample $S$,
the learner produces a {\em hypothesis} $h:\X\to\set{-1,1}$,
whose {\em empirical error} is defined by $\serr(h)=n\inv\sum_{i=1}^n\pred{h(X_i)\neq Y_i}$
and whose {\em generalization error} is defined by $\err(h)=\P(h(X)\neq Y)$.
The Bayes-optimal classifier, $h\bay$, is defined by 
$$h\bay(x)=\argmax_{y\in\set{-1,1}}\P(Y=y\gn X=x)$$
and 
$$R^*:=\err(h\bay)=\inf\set{\err(h)},$$ where the infimum is over all measurable hypotheses.
A learning algorithm mapping a sample $S$ of size $n$ to a hypothesis $h_n$ is said to be
strongly Bayes consistent if $\err(h_n)\ninf R^*$ almost surely.

\paragraph{Sub-sample, margin, and induced $1$-NN.} 
In a slight abuse of notation, we will blur the distinction
between $S\subset\X$ as a collection of points in a metric space
and $S\in (\X\times\set{-1,1})^n$ as a sequence
of labeled examples.
Thus, the notion of a {\em sub-sample} $\tilde S\subset S$
partitioned 
into its positively and negatively labeled subsets as
$\tilde S=\tilde S_+\cup \tilde S_-$ is well-defined.
The {\em margin} of $\tilde S$, defined by
$$\marg(\tilde S)=\dist(\tilde S_+,\tilde S_-),$$
is the minimum distance between a pair
of opposite-labeled points (see Fig.~\ref{fig:marg}).
A sub-sample $\tilde S$ naturally induces the $1$-NN classifier
$h_{\tilde S}$, via
\beq
h_{\tilde S}(x) = \sign(\dist(x,\tilde S_-)-\dist(x,\tilde S_+)).
\eeq

\paragraph{Margin risk.} 
For 
a given sample $S$ of size $n$,
any $\g>0$ and measurable $f:\X\to\R$,
we define
the {\em margin risk}
$$\risk_\g(f) = \P( Y f(X) < \g )$$
and
its empirical version
$$\emprisk_{n,\g}(f) = \oo n\sum_{i=1}^n \chr_{\{Y_i f(X_i) < \g\}}.$$
When $\g=0$, we omit it from the subscript;
thus, e.g., $\risk(f) = \P( Y f(X) < 0 )$,
which
agrees
with 
our previous
definitions of $\risk(h)$ and $\emprisk_n(h)$ 
for binary-valued $h$.
\tikzset{
	plusS/.style = {  },
	plustS/.style = {double = black!30, double distance=1pt},
	minusS/.style = {},
	minustS/.style = {double = black!30, double distance=1pt}
}
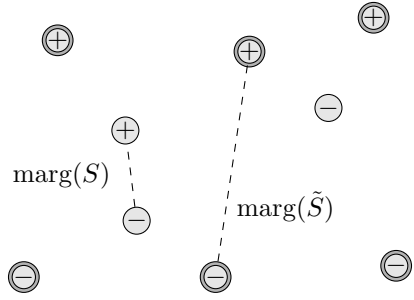
\begin{figure}%
\centering
\begin{tikzpicture}[scale = 1.5,
every node/.style = {draw, shape = circle, minimum size = 2.1mm,inner sep = 0mm, fill = black!10}
]
\node  [plusS] (p1) at (-1.1,-.2) {$+$};
\node  [minusS] (m1) at (-1,-1) {$-$};
\node  [plustS] (p2) at (.0, .5) {$+$};
\node  [minustS] (m2) at (-.3,-1.5) {$-$};
\node  [plustS] (ps) at (-1.7,.6) {$+$};
\node  [plustS] (ps3) at (1.1,.8) {$+$};
\node  [minustS] (mts) at (1.3,-1.4) {$-$};
\node  [minusS] (ms) at (.7,0) {$-$};
\node  [minustS] (ms) at (-2,-1.5) {$-$};
\draw  [dashed] (m1) -- (p1) node [left=7pt, midway, draw=none, fill = none] {$\marg(S)$};
\draw  [dashed](p2) -- (m2) node [above right=18pt,draw=none, fill = none] {$\marg(\tilde S)$};
\end{tikzpicture}
\caption{In this example, the sub-sample $\tilde S\subset S$ is indicated by double circles.
It is always the case that
$\marg(\tilde S)\ge\marg(S)$.}%
\label{fig:marg}%
\end{figure}

\section{Learning Algorithm: Regularized $1$-NN}
\label{sec:learn-alg}
This section is provided to cast known results 
(or their minor modifications)
in the terminology of this paper. 
As the 
main contribution of this paper
is a Bayes-consistency analysis of 
a particular learning algorithm, we must first provide the details of the latter.
The learning algorithm 
in question
is essentially the one
given in \citet{DBLP:conf/colt/GottliebKK10}.
Our point of departure is the connection made by
\citet{DBLP:journals/jmlr/LuxburgB04} between Lipschitz functions and
$1$-NN classifiers.
\begin{theorem}[\citet{DBLP:journals/jmlr/LuxburgB04}]
\label{thm:nn-lip}
If $\tilde S$ is a sub-sample 
with 
$\marg(\tilde S)\ge\g$,
then there is an 
$f\in\F_{2}$
such that 
$$
h_{\tilde S}(x)=\sign(f(x))
$$
for all $x\in\X$.
More explicitly, 
$f\in\F_{2}$
is a {\em Lipschitz extension}
of $\tilde S$, satisfying 
\beqn
\label{eq:lip-ext}
f(x) = f_{\tilde S}(x) =
\begin{cases} 
+\g
, &\mbox{if }  x\in \tilde S_+
\\
-\g, & \mbox{if } x\in \tilde S_-
. 
\end{cases} 
\eeqn
\end{theorem}

We will only consider members of $\F_{2}$
that are Lipschitz-extensions of $\g$-separable sub-samples
and will never need to actually calculate these explicitly;
their only purpose is to facilitate the analysis.
In line with the Structural Risk Minimization (SRM) paradigm,
our learning algorithm consists of minimizing the {\em penalized margin risk},
\beqn
\label{eq:Rdef}
\emprisk^\srm_{n,\g}(f) &=&
\emprisk_{n,\g}(f)
+ r^\srm(n,\g),
\eeqn
where
\beqn
\label{eq:rsrmdef}
r^\srm( n, \g )  &=&
\frac{4}{\g}
\paren{\frac{c_{\dim} }{n}}^{\frac{1}{2(\dim+1)}}
\\\nonumber& +&
\sqrt{\frac{\frac{c_1}{\dim+1} \log\paren{n/c_\dim} + 2 c_1\log \log \frac{2e}{\g} }{n}}
\eeqn
and $c_1$, $c_\dim$ are explicitly computable constants, the latter depending only on $\dim$.
The form of the penalty term (which is different from the penalty
term in~\citet{DBLP:conf/colt/GottliebKK10})
will be motivated by the analysis in the sequel.

This optimization is performed via two nested routines:
the inner one minimizes
$\emprisk_{n,\g}^\srm(f)$ over $f\in\F_{2}$  for a fixed $\g$,
while the outer one minimizes over $\g>0$.
Since this is a very slight modification of the SRM procedure proposed and analyzed in
\citet{DBLP:conf/colt/GottliebKK10}, we will give a high-level sketch.

\paragraph{Inner routine: optimizing over $f\in\F_{2}$.}
By Theorem~\ref{thm:nn-lip}, 
minimizing $\emprisk_{n,\g}^\srm(f)$ over $f\in\F_{2}$  for a fixed $\g$
is equivalent to seeking a $\g$-separable $\tilde S\subset S$
whose induced $1$-NN classifier $h_{\tilde S}$ makes the fewest mistakes on $S$
(see Algorithm~\ref{alg:inner}).
The algorithm invokes a minimum vertex cover routine, which
by K\"onig's theorem
is equivalent to maximum matching
for bipartite graphs, and is computable in randomized time $O(n^{2.376})$
\citep{1033180}.

\begin{algorithm}
\caption{
minimizing $\emprisk_{n,\g}^\srm(f)$ over $f\!\in\!\F_{2}$ for fixed $\g$
}
\label{alg:inner}
\begin{algorithmic}[1]
\Function{Inner}{$S$,$\g$}
\State construct bipartite graph $G=(S_+,S_-,E)$ with 
$$E=\set{(x,x'):x\in S_+,~x'\in S_-,~\dist(x,x')<\g}$$
\State compute minimum vertex cover 
$C$
for $G$
\State\Return $\tilde S=S\setminus C$
\EndFunction
\end{algorithmic}
\hide{}
\end{algorithm}

\paragraph{Outer loop: minimizing over $\g>0$.}
Although 
$\g$ takes on 
a continuum of 
values, we need only consider those induced by distances between opposite-labeled points in $S$, of which there are $O(n^2)$. For each candidate $\g$, Algorithm~\ref{alg:inner}
computes the optimal $ f_{n,\g}^*\in\F_{2}$.
Let $\g_n^*$ be a minimizer of ${\emprisk}_{n,\g}^\srm( f_{n,\g}^*)$,
with corresponding $ f_{n}^*:= f_{n,\g_n^*}^*$:
\beqn
\label{eq:Rsrm}
\begin{array}{rcl}
\emprisk_{n,*}^\srm 
&:=& {\ds \inf_{\g > 0}\inf_{f\in\F_{2}}{\emprisk}_{n,\g}^\srm(f) 
\phantom{\sum_{a_a}}
}
\\[3pt]
& = & {\ds \inf_{\g > 0}{\emprisk}_{n,\g}^\srm(f_{n,\g}^*) }
\\[10pt]
&= & { \emprisk_{n,\g_n^*}^\srm(f^*_{n})}
.
\end{array}
\eeqn

The total runtime for computing
$\g_n^*$ 
and
$f_n^*$
is $O(n^{4.376})$, which may be considerably sped up
if one is willing to tolerate a small approximation factor
\citep{DBLP:conf/colt/GottliebKK10,gottlieb2014efficient}.

\section{Consistency proof}
\label{sec:consist-proof-roi}

\begin{table*}[t]%
\centering
\begin{tabular}{l|l|l|c}
symbol 
& meaning
& formally
& Eq.  
\\
\hline
$\risk_{\g}{(f)}$  
& \parbox[c][17pt]{150pt}{$\g$-margin risk}
& $ \P( Y f(X) < \g )$
&  

\\[3pt]
$\emprisk_{n,\g}(f)$  
&  \parbox[c][17pt]{150pt}{empirical $\g$-margin risk}
& $\oo n\sum_{i=1}^n \chr_{\{Y_i f(X_i) < \g\}} $
&  

\\[3pt]
$\emprisk_{n,\g}^\srm(f)$  
&  \parbox[c][17pt]{150pt}{penalized  empirical $\g$-margin risk}
& $\emprisk_{n,\g}(f) + r^\srm(n,\g)$
& (\ref{eq:Rdef},\ref{eq:rsrmdef})


\\[3pt]
$\emprisk_{n,*}^\srm$  
&  \parbox[c][17pt]{150pt}{optimal penalized  empirical risk} 
& $\inf_{\g > 0}\inf_{f\in\F_{2}}{\emprisk}_{n,\g}^\srm(f)$
&   (\ref{eq:Rsrm})

\\[3pt]
$f_{n,\g}^*$  
&  \parbox[c][17pt]{150pt}{optimal $f\in \F_2$ for a fixed $\g$}
&  ${\emprisk}_{n,\g}^\srm(f_{n,\g}^*) = \inf_{f\in\F_{2}}{\emprisk}_{n,\g}^\srm(f)$
&  (\ref{eq:Rsrm})

\\[3pt]
$\g_n^*$, $f_n^*$  
&  \parbox[c][17pt]{150pt}{optimal margin and optimal $f\in \F_2$
}
&  $\emprisk_{n,*}^\srm = \emprisk_{n,\g_n^*}^\srm(f_n^*)$
&   (\ref{eq:Rsrm})

%
\\[3pt]
$\loss_{\g,\gap}(f)$  
&  \parbox[c][17pt]{150pt}{surrogate risk} 
& $\E\big[ \sloss_{\g,\gap}(Y f(X) )\big]$
& (\ref{eq:surrisk})

\\[3pt]
$\empsrisk_{n,\g,\gap}(f)$  
&  \parbox[c][17pt]{150pt}{empirical surrogate risk} 
& $\oo{n}\sum_{i=1}^n \sloss_{\g,\gap}(Y_i f(X_i))$
& (\ref{eq:surrisk}) 
\\[3pt]
\hline
\end{tabular}
\caption{A summary of the notation.
}
\label{tab:notation}
\end{table*}

We now
prove the main 
technical
result of the paper:
\begin{theorem}
\label{thm:main} With probability 
one 
over the random sample $S$
 of size $n$,
$$
\lim_{n \to \infty} \risk(f_n^*) = \risk^* .
$$
\end{theorem}

We will 
break it up into high-level steps.
The basic plan is standard:
decompose the excess risk into two terms,
\beqn
\label{eq:decomposeLip}
\nonumber
\risk(f_n^*) - \risk^* 
&=&
\Big( 
\risk(f_n^*)  - \emprisk_{n,*}^\srm
\Big)
 +  
\Big(  
\emprisk_{n,*}^\srm - \risk^*
\Big)
\\&=&
\mbox{(I) \,+\, (II)}
,
\eeqn
and show that each decays to $0$ almost surely.
For convenience, the notation used in the proof is summarized in Table \ref{tab:notation}.
All omitted proofs are given in the Appendix.

\subsection{The term (I)}
\label{subsec:termI}
In order to connect $\emprisk_{n,*}^\srm$ and $\risk(f_n^*)$ 
we first need  a concentration bound.
More specifically, since $\emprisk_{n,*}^\srm$ involves the optimal margin $\g_n^*$ (which is a priori unknown),
we would like to prove 
for each $\g>0$
a deviation estimate on
$$\smabs{\risk_{\g}(f)  - \emprisk_{n,\g}(f)},$$
uniformly over all $f\in\F_2$.
We find it most convenient to do this using Rademacher complexities\footnote{
An alternative, though somewhat messier route, would be to use fat-shattering dimension, as in
\citet{DBLP:conf/colt/GottliebKK10}.
},
but these 
require a loss that is Lipschitz-continuous in $\g$ --- and
$\emprisk_{n,\g}(f)$ is not even continuous (it is lower-semicontinuous in $\g$ for a fixed $f$).
We overcome this technical hurdle by introducing a surrogate loss 
$\sloss_{\g,\gap}$ and corresponding surrogate risk
$\srisk_{\g,\gap}$
as follows.
\paragraph{Surrogate loss.} 
For $0<\g,\gap\leq 1$ define the {\em surrogate loss} function $\sloss_{\g,\gap}(u):\R\to[0, 1]$
\beqn
\label{eq:surloss}
\sloss_{\g,\gap}(u) = 
\left\{
\begin{array}{ll}
	1 & \text{if } u \leq \g(1-\gap),\\
	0 & \text{if } u \geq \g,\\
	{(\g -  u)}/{(\g\gap)} & \text{otherwise},
\end{array}
\right.
\eeqn
illustrated
in Figure \ref{fig:surr-loss},
and its associated expected and empirical surrogate risks,
\beqn
\label{eq:surrisk}
\begin{array}{ccl}
\srisk_{\g,\gap}(f) 
& = & {\ds \E\big[ \sloss_{\g,\gap}(Y f(X) )\big] }
,\\[6pt]
\empsrisk_{n,\g,\gap}(f) 
& = & {\ds \oo{n}\sum_{i=1}^n \sloss_{\g,\gap}(Y_i f(X_i))}
.
\end{array}
\eeqn
At this point, it appears as though we have two free parameters: $\g$ and $\gap$.
However, we will tie them together via 
a common (double) stratification scheme.
For $n,l\in\N$ 
put
\begin{align}
\label{eq:grid-L}
{\g_{n,l}} & =  \paren{1-\gap_{n}}^{l-1}
,\qquad
\gap_{n}  =  {1}/{\nd}
\\[5pt]
\label{eq:grid-eps}
\eps_{n,l}
& = 
\frac{2}{\g_{n,l} \gap_{n} \nd^2}%
 \cc+\cc
 \sqrt{\cc\frac{{2 c_1}{}\log\paren{\oo{\gap_n} \log \frac{e}{\g_{n,l}}}  }{n}}\cc,
\end{align}
where
\beqn
\nd = \paren{\frac{n}{c_\dim}}^{\frac{1}{2(\dim+1)}}.
\eeqn

This enables us to obtain a uniform deviation estimate:
\begin{lemma}
\label{lem:stratifyLip}
For all 
$n\in\N$ and
$\eps > 0$,
\beq
\P\Big( \exists l\in\N:
\cc\cc\sup_{f \in \F_{2}} 
		\Big|
				\srisk_{\g_{n,l},\gap_{n}}(f) 
				- \empsrisk_{n,\g_{n,l},\gap_{n}}(f)
				\Big|
		> \eps  +  \eps_{n,l}
\Big)\,\,\\
\leq
 \frac{\pi^2}{6} \exp\paren{-\frac{n\eps^2}{c_1}}.			
\eeq
\end{lemma}

Armed with this uniform deviation bound,
we proceed with the proof that the term (I) decays to zero almost surely.
By Theorem \ref{thm:nn-lip} we may 
assume that $f_n^*\in\F_2$ is 
in the form of (\ref{eq:lip-ext})
 %
with $\g = \g_n^*$
being the optimal margin.
Given $\g_n^*$,
let $l_n^-,l_n^+ \in \N$ be the \emph{consecutive} 
margin indexes in the stratification grid (\ref{eq:grid-L})
such that
\beq
\forall n\in\N
,\quad
\g_n^* \in [\g_{n,l_n^-},\g_{n,l_n^+})
,\qquad
l_n^-  =  l_n^+ + 1 
\eeq
and abbreviate
$
\g_n^+ = \g_{n,l_n^{+}}
$
and
$
\g_n^- = \g_{n,l_n^{-}}.
$
We now
relate the margin risks to the surrogate risks.
Note that since $0 \leq \g_n^- \leq \g_n^* \leq \g_n^+ $, we have
\beq
\risk(f_n^* )
& \leq & 
\srisk_{\g_n^-, \gap_n}(f_n^* ),
\\
\emprisk_{n,\g_n^*}(f_n^* ) 
& \geq &
\empsrisk_{n,\g_n^-,\gap_n}(f_n^* ),
\\
r_\srm(n,\g_n^*) 
&\geq&
 r_\srm(n,\g_n^+).
\eeq
Thus,
\beq
\pstar	& := &
		\P\Big(
					\risk(f_{n}^*) 
						-  \emprisk_{n,*}^\srm
 				> \eps 					 					
		\Big)
\\
& = &
\P\Big(
{\risk(f_{n}^*)  
							- 
							 \emprisk_{n,\g_n^{*}}(f_{n}^*)} 	-   r_\srm(n,\g_n^{*})
							> \eps  
\Big)
.
\\
& \leq &
\P\Big( %
{\srisk_{\g_n^-,\gap_n}(f_{n}^*)  
							- 
							 \empsrisk_{n,\g_n^-, \gap_n}(f_{n}^*)} 
							\\&& \hspace{3pt} > \eps  + r_\srm(n,\g_n^{+})	
\Big),
\eeq
and since $f_n^*\in\F_2$, we have
\beq
\pstar 
& \leq &
\P\Big( 
			\sup_{f\in \F_{2}}
\abs{\srisk_{\g_n^-,\gap_{n}}(f)  
							- 
							 \empsrisk_{n,\g_n^-,\gap_{n}}(f)} 
							\\&&  \hspace{3pt} > \eps  +  r_\srm(n,\g_{n}^+)	
\Big)
\\
& \leq &
\P\Big( \exists l\geq 2 :
			\sup_{f\in \F_{2}}
\abs{\srisk_{\g_{n,l},\gap_{n}}(f)  
							- 
							 \empsrisk_{n,\g_{n,l},\gap_{n}}(f)} 
							\\&&  \hspace{3pt} > \eps  +  r_\srm(n,\g_{n,l-1})
\Big)
.
\eeq
Next, we make
a connection
 between $r_\srm(n,\g_{n,l-1})$ and $\eps_{n,l}$,
justifying
the form of the penalty term in
(\ref{eq:rsrmdef}):
\begin{lemma}
\label{lem:rsrm-to-eps}
For all $l\geq2$
and all $n$ sufficiently large,
$$
r_\srm(n,\g_{n,l-1}) \geq \eps_{n,l} .
$$
\end{lemma}
An application of Lemma \ref{lem:rsrm-to-eps} yields
\beq
\pstar
& \leq &
\P\Big( \exists l\geq2 :\cc\cc
			\sup_{f\in \F_{2}}
\abs{\srisk_{\g_{n,l},\gap_{n}}(f)  
							- 
							 \empsrisk_{n,\g_{n,l},\gap_{n}}(f)} 
							\\&&  \hspace{3pt} > \eps  +  \eps_{n,l}
\Big)
\\
&  \leq & \frac{\pi^2}{6}\exp\paren{-\frac{n\eps^2}{c_1}},
\eeq
where the last inequality follows from Lemma \ref{lem:stratifyLip}.

\begin{figure}[t]
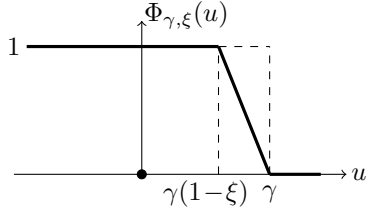

\centering
\tikz[scale = 1.7]{
\draw [->] (-1,0) -> (1.6,0);
\draw [->] (0, 0) -> (0, 1.2);
\draw [very thick] (-.9,1) -> (.6,1);
\draw [very thick] (.6,1) -> (1,0);
\draw [very thick] (1,0) -> (1.4,0);
\draw [dashed] (.6,0) -> (.6,1);
\draw [dashed] (1,0) -> (1,1);
\draw [dashed] (.6,1) -> (1,1);
\draw [fill] (0,0) circle   (1pt);
\node at (1,-.15) {$\g$};
\node at (-1,1) {$1$};
\node at (.5, -.15) {$\g(1\cc-\cc\gap)$};
\node at (1.7,0) {$u$};
\node at (.35,1.25) {$\sloss_{\g,\gap}(u) $};
}
\caption{The surrogate loss function.}
\label{fig:surr-loss}
\end{figure}

\subsection{The term (II)}
We begin by approximating the Bayes optimal risk by the margin risk:
\begin{lemma}
\label{lem:denseLip}
For every $\eps>0$ there is a $\g>0$ such that
\beq
\inf_{f\in \F_2}\risk_{{\g}}(f) - \risk^* < \eps.
\eeq
In particular,
\beqn
\label{eq:denseLip}
\risk^*=\lim_{\g\to0}\inf_{f\in \F_2}\risk_{{\g}}(f).
\eeqn
\end{lemma}

Since (\ref{eq:denseLip}) holds for any sequence $\g_n\ninf0$,
it is true in particular of subsequences of
the
stratification grid
(\ref{eq:grid-L}).
Hence, for all $\eps>0$, there is a
$\tilde \g^+$ 
with a corresponding $\tilde{f}^+\in\F_2$ such that
\beq
\inf_{f\in \F_2}\risk_{\tilde{\g}^+}(f) & \leq &  \risk^* +  \eps/8,
\\
\risk_{\tilde{\g}^+}(\tilde{f}^+) & \leq & \inf_{f\in \F_2}\risk_{\tilde{\g}^+}(f) + \eps/8
.
\eeq
Fix such a $\tilde \g^+$ and let $\tilde \g^-$ be the 
``next''
margin in the stratification (\ref{eq:grid-L}).
Now by (\ref{eq:Rsrm}),
Algorithm~\ref{alg:inner} provides an optimal $f_n^*$ such that
\beq
\emprisk_{n,*}^\srm
\;=\; \emprisk_{n,\g_n^*}^\srm(f_n^*) 
\;\leq\; \emprisk_{n,\tilde{\g}^-}^\srm(f_{n,\tilde{\g}^-}^*)
\;\leq\; \emprisk_{n,\tilde{\g}^-}^\srm(\tilde{f}^+).
\eeq
Hence,
for the term (II) we have
\beq
\psstar    & := & 	\P\Big(
 		\emprisk_{n,*}^\srm
		- \risk^*
		> \eps
	 \Big)\\
& \leq &
 	\P\Big(
 		\emprisk_{n,*}^\srm
		- \risk_{\tilde{\g}^+}(\tilde f^+)
		> 3\eps/4
	 \Big)
\\
& \leq &
 	\P\Big(
 		\emprisk_{n,\tilde{\g}^-}(\tilde{f}^+)
		- \risk_{\tilde{\g}^+}(\tilde f^+)
		\\
		& & \hspace{27pt} > 3\eps/4 - r^\srm(n,\tilde \g^-)
	 \Big)
.
\eeq
Next, note that
the margin loss
$\risk_{\tilde{\g}^+ }(\cdot)$ is well-approximated by surrogate losses:
\begin{lemma} For every $\g>0$ and $f\in\F_2$
\label{lem:approx}
\beqn 
\label{eq:lim-eta-cutoff}
\lim_{n\to\infty} \abs{\srisk_{\g,\gap_{n}}(f) - \risk_{\g}(f)} & = & 0.
\eeqn
\end{lemma}

Hence,
we may take $n$ sufficiently large so that
\beq
\abs{ \srisk_{\tilde{\g}^+,\gap_{n}}(\tilde f^+)- \risk_{\tilde{\g}^+}(\tilde f^+)} & \leq & \eps/4.
\eeq
Since by construction,
\[
\frac{\g_{n,l+1}}{\g_{n,l}} = 1 -\gap_{n}
,\qquad
\forall l\in\N,
\]
it follows that $\tilde\g^- = \tilde\g^+(1-\gap_n)$ and thus
$$
\emprisk_{n,\tilde{\g}^- }(\tilde{f}^+) 
\leq
\empsrisk_{n, \tilde{\g}^+,\gap_{n}}(\tilde f^+).
$$
Taking $n$ sufficiently large to ensure
$r^\srm(n,\tilde \g^-) \leq \eps/4$ 
and
combining these estimates yields
\begin{align*}
\psstar
	 & \,\leq\,  
	 \P\Big(
 		\empsrisk_{n,\tilde{\g}^+,\gap_{n}}(\tilde{f}^+)
		- \srisk_{\tilde{\g}^+,\gap_{n}} (\tilde{f}^+)
		> \eps/4 
	 \Big)
	 \\
	 & \,\leq\, 
	 \P\Big(
 		\sup_{f\in \F_{2}}
 		\abs{\empsrisk_{n,\tilde{\g}^+,\gap_{n}}({f})
		- \srisk_{\tilde{\g}^+,\gap_{n}} ({f})
		}
		> \eps/4 
	 \Big)
\\& \,\leq\, c e^{-\frac{n\eps^2}{16 c_1}},
\end{align*}
analogously to the bound on term (I).

\section{Experiments}
\label{sec:exp}
\begin{figure}
\begin{center}
\includegraphics[width=.6\columnwidth]{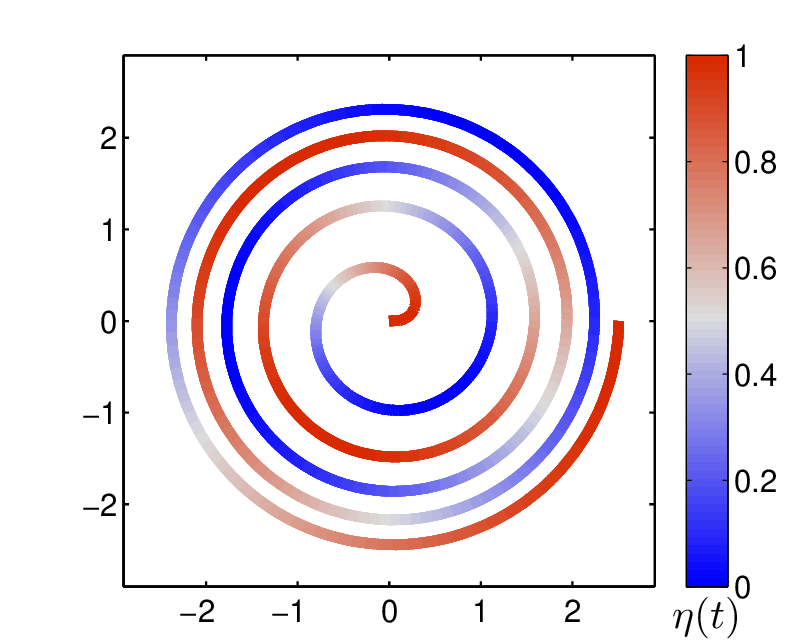}%
\caption{The distribution for $A = 5$ and $\omega = 3$.}
\label{fig:spiral}%
\end{center}
\end{figure}
We ran simulations with a twofold purpose: 
(a) to ascertain the convergence of various classifier risks to the Bayes optimal risk and to
compare their rates of convergence and (b) to compare the actual runtimes of the various algorithms.
To this end, we took 
$\X = \R^2$
endowed with the Euclidean metric
$\dist(x,x') = \norm{x-x'}_2$, and defined 
a joint distribution
over $\X \times\{-1,1\}$
as follows.
A point $(x_1,x_2)\in\R^2$ is sampled by drawing
$T\in[0,2\pi]$ 
uniformly at random and putting
\beq
x_1(T) &=& A \sqrt{T} \cos( \omega T),
\\
x_2(T)  &=& A \sqrt{T} \sin( \omega T)
\eeq
for some specified parameters $A$ and $\omega$.
The label $Y\in\{-1,1\}$ is drawn 
according to the conditional distribution
\beq
\eta(T) = \P( Y =1 \gn T ) = \frac{1 + \cos( \omega T )}{2},
\eeq
as illustrated in Figure~\ref{fig:spiral}.

We compared four classifiers:
$k^*$-NN (the $k$-NN classifier with $k$ optimized by cross-validation),
SVM (support vector machine with the RBF kernel whose bandwidth and regularization
penalty were optimized by cross-validation),
CV-1-NN (margin-regularized 1-NN with $\g$ tuned by cross-validation),
and SRM-1-NN (the 1-NN classifier described in Section~\ref{sec:learn-alg} using a greedy
vertex cover heuristic rather than the exact matching algorithm while searching for the optimal margin).
Their runtime and generalization performance, averaged over $100$ independent runs, are summarized in 
Figures \ref{fig:ErrVsNumSamp} and \ref{fig:RunVsNumSamp}.

Our proposed algorithm, SRM-1-NN, emerges competitive by both criteria.

\begin{figure}
\centerline{\includegraphics[width=.5\columnwidth]{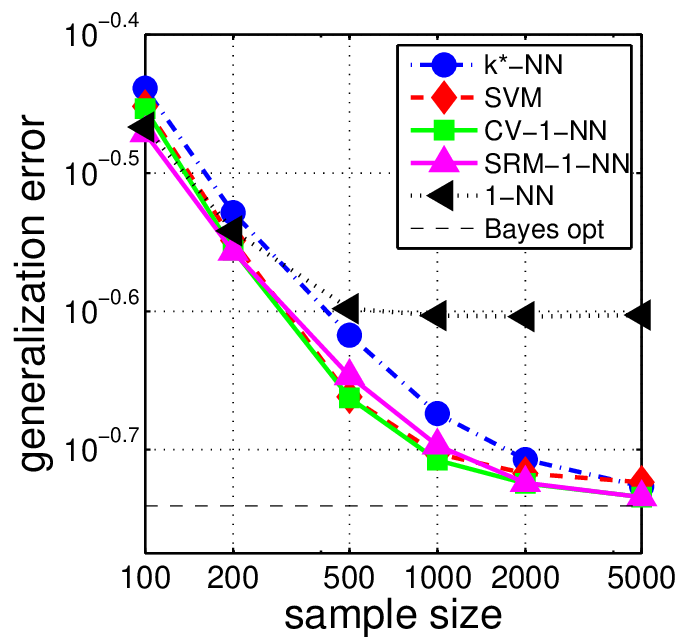}}
\caption{Generalization error vs. number of samples.
CV-1-NN is uniformly dominant, but for large sample sizes
SRM-1-NN catches up.
Unregularized
$1$-NN is included for reference; it is clearly not Bayes consistent.
}
\label{fig:ErrVsNumSamp}%
\end{figure}

\begin{figure}
\centerline{\includegraphics[width=.4\columnwidth]{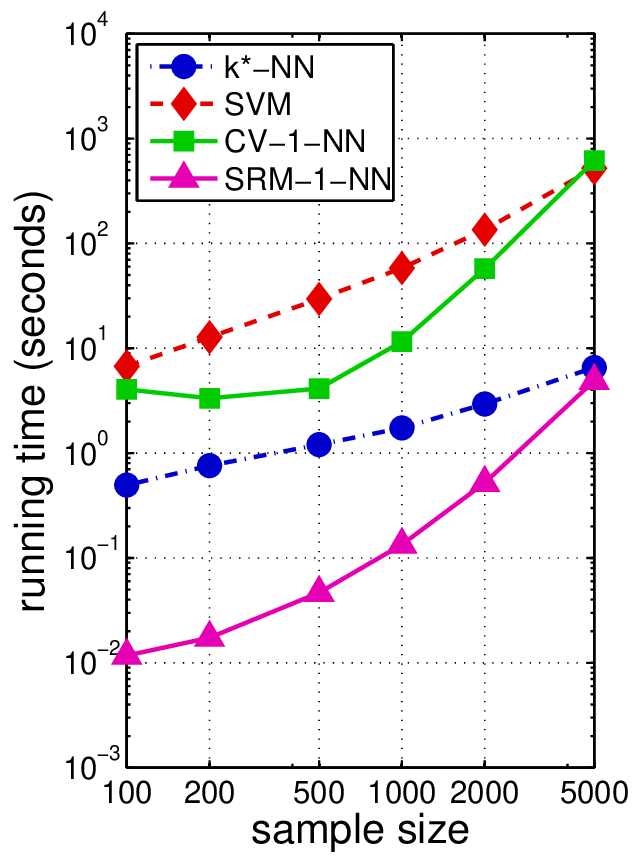}}
\caption{Running time 
vs. number of samples.
SRM-1-NN enjoys a clear time advantage over the other methods involving cross-validation.
}
\label{fig:RunVsNumSamp}%
\end{figure}

\appendix
\section{Appendix}
\subsection{Proof of Lemma \ref{lem:stratifyLip}}
We first need the following uniform convergence lemma.
\begin{lemma}
\label{lem:cov-num}
For any $0  < \eps$, $0 < \gap < 1$ and  $0 < \g$, 
\begin{align}
\label{eq:Mohri-radem-thm}
& \P\Big(\sup_{f\in\F_2}
\abs{\srisk_{\g,\gap}(f)-\empsrisk_{n,\g,\gap}(f)}
\\
\nonumber
& \hspace{40pt}
> 2\radem_n(\srisk_{\g,\gap} \circ \F_2) + \eps \Big)
\le \exp\paren{-n\eps^2/c_1},
\end{align}
where the Rademacher complexity $\radem_n(\srisk_{\g,\gap} \circ \F_2)$ 
satisfies
\beqn
\label{eq:rade-lip}
\radem_n(\srisk_{\g,\gap} \circ \F_2)
\,\leq \,
\frac{2}{\g\gap}
\paren{\frac{c_{\dim}}{n}}^{\frac{1}{\dim+1}}
\;=:\;
\radem_{n,\g,\gap}
.
\eeqn
\end{lemma}

\begin{proof}[Proof of Lemma \ref{lem:cov-num}]
Equation (\ref{eq:Mohri-radem-thm}) is restatement of 
\citet[Theorem 3.1]{mohri-book2012}.
Note that $\sloss_{\g,\gap}:\R\to[0,1]$ is ${1/(\g\gap)}$-Lipschitz.
Thus, by Talagrand's contraction lemma \citep{LedouxTal91},
\beq
\radem_n(\srisk_{\g,\gap} \circ \F_2) \leq \frac{2}{\g\gap}\radem_n({\F_1}).
\eeq
The 
upper estimate on $\radem_n({\F_1})$
implicit in (\ref{eq:rade-lip})
is essentially 
contained in 
Equation (10)
of
\citet{kon-weiss-2014}.
\end{proof}

\begin{proof}[Proof of Lemma \ref{lem:stratifyLip}]
Following proof idea in 
\citet[Theorem 18.2]{MR1383093},
a union bound yields
\begin{align*}
\pssstar  \;:=\; &
 \P\Big( \exists l\in\N:
\sup_{f \in \F_{2}} 
		\Big|
				{\srisk_{\g_{n,l},\gap_{n}}(f) 
				- \empsrisk_{n,\g_{n,l},\gap_{n}}(f)} 
		\Big|
		\\
	& 
	 \hspace{10pt}
	 > \eps  +  \eps_{n,l}
\Big)
\\
\; \leq \; &
 \sum_{l=1}^{\infty  }
\P\Big( 
\sup_{f \in \F_{2}} 
		\Big| 
				{\srisk_{\g_{n,l},\gap_{n}}(f) 
				- \empsrisk_{n,\g_{n,l},\gap_{n}}(f)} 
		\Big|
\\
 &
\hspace{10pt}>
		\eps  +  \eps_{n,l}
\Big).
\end{align*}
Note that by construction 
\beq
\eps_{n,l} 
& = & 2\radem_{n,\g_{n,l},\gap_{n}} + \sqrt{\cc\frac{{2 c_1}\log\paren{\oo{\gap_n} \log \frac{e}{\g_{n,l}}}  }{n}}\cc
.
\eeq
Thus, writing 
\[
r_{n,l} = \sqrt{\cc\frac{{2 c_1}\log\paren{\oo{\gap_n} \log \frac{e}{\g_{n,l}}}  }{n}}\cc
\]
and bounding each term in the sum by Lemma \ref{lem:cov-num}, we have
\beq
 \pssstar &\leq &
\sum_{l=1}^{\infty} 
 \exp\paren{
 		-\frac{n(\eps  +  r_{n,l}
 		)^2 } {c_1}
 		}
\\
&\leq&
\exp\paren{-\frac{n\eps^2}{c_1}}
\sum_{l=1}^{\infty} 
\exp\paren{-\frac{n r_{n,l}
	^2}{c_1}}
.
\eeq
Next note that by the definition of $\g_{n,l}$
we have
\[
\g_{n,l} = (1-\gap_n)^{l-1} \leq \exp\paren{-\gap_n (l-1)}.
\]
Solving for $l$ yields
\[
\oo{\gap_n}\log \frac{e}{\g_{n,l}}  \geq  l.
\]
Thus,
\beq
\exp\paren{-\frac{n r_{n,l}
	^2}{c_1}} &\leq& \frac{1}{l^2}
\eeq
and
summing over $l$ yields the claim.
\end{proof}

\subsection{Proof of Lemma \ref{lem:rsrm-to-eps}}
\label{app:det-pen}
Let us first write $\eps_{n,l}$ in terms of $\g_{n,l-1}$. 
Since $\g_{n,l} = \g_{n,l-1}(1-\gap_{n})$ by definition,
we have
\beq
\eps_{n,l} & = &
\frac{2}{\g_{n,l-1} (1-\gap_{n}) \gap_{n} \nd^2}\\
 && +
 \sqrt{\frac{{2 c_1}{}\log\paren{\oo{\gap_n} \log \frac{e}{\g_{n,l-1}(1-\gap_{n})}}  }{n}}.
\eeq 
Taking $n$ sufficiently large to ensure that $1-\gap_{n} \geq 1/2$
we have that
for all $l \geq 2$,
\beq
\eps_{n,l}  \leq 
\frac{4}{\g_{n,l-1} \gap_{n} \nd^2}
  +
 \sqrt{\cc\frac{{2 c_1}{}\log\paren{\oo{\gap_n} \log \frac{2e}{\g_{n,l-1}}}  }{n}}\cc
,
\eeq
which is exactly $r^\srm(n,\g_{n,l-1})$.

\subsection{Proof of Lemma \ref{lem:denseLip}}
\newcommand{\hf}{{\textstyle\oo2}}
The function $\eta:\X\to[0,1]$ given by $\eta(x)=\P(Y=1|X=x)$
is measurable \cite[Corollary B.22]{MR1354146}
and hence, by virtue of being bounded, belongs to $L_1(\mu)$,
where $\mu$ is the marginal distribution over $\X$.

Now
\beq
R^*
&=&
\P(Y(\eta(X)-\hf)\le0)
\\&=&
\lim_{k\to\infty}
\P\paren{Y(\eta(X)-\hf)<1/ k}
\\&=&
\lim_{k\to\infty}
\P\paren{kY(\eta(X)-\hf)<1}
\\&=&
\lim_{k\to\infty}
\P(g_k(X,Y)<1)
,
\eeq
where
$g_k(x,y):=
  ky(\eta(x)-\hf)
  \in L_1(\mu)$
and
we invoked Lebesgue's dominated convergence theorem
together with the fact that
$\chr[\alpha<\oo k]\longto_{k\to\infty}\chr[\alpha\le0]$
for all $\alpha\in\R$.

We also observe that $R_\gamma(f)=\P(Yf(X)<\gamma)=r$ for some $f\in\F_2$
if and only if there is an $\tilde f\in\F_{2/\gamma}$ for which
$r=\P(Y\tilde f(X)<1)$.
Define the metric $\tilde\rho$ on $\X\times\Y$ by
$\tilde\rho((x,y),(x',y'))=\rho(x,x')+\chr[y\neq y']$
and denote the collection of all $L$-Lipschitz functions
on
$(\X\times\Y,\tilde\rho)$
by $\tilde\F_L$.
The compactness of
$(\X\times\Y,\tilde\rho)$
is inherited from $(\X,\rho)$ and each $f\in\bigcup_{L\ge0}\F_L$
has $\nrm{f}_\infty<\infty$.
Since for $f\in\F_L$,
\beq
yf(x)- y'f(x')
&=&
yf(x)- yf(x')
+yf(x')- y'f(x')
\\&\le& \abs{f(x)-f(x')}
+2\nrm{f}_\infty\chr[y\neq y']
\\&\le& \max\set{L,2\nrm{f}_\infty}\tilde\rho((x,y),(x',y')),
\eeq
and conversely, for $yf(\cdot)\in\tilde\F_L$,
\beq
|f(x) - f(x')| = |yf(x)-yf(x')| \le L,
\eeq
it follows that $f\in\bigcup_{L\ge0}\F_L$ if and only if
$yf(\cdot)\in\bigcup_{L\ge0}\tilde\F_L$.

We claim that the collection of all Lipschitz functions, $\bigcup_{L\ge0}\tilde\F_L$
is dense in $L_1(\mu)$.
Indeed, Theorem~\ref{thm:cc-dense} below shows that
the continuous functions are dense in $L_1(\mu)$,
and these can be uniformly approximated by Lipschitz ones
in our case \cite{MR0214994,miculescu2000}.
In particular, given our assumptions on $\X$ and $\mu$,
it follows that for all $g\in L_1(\mu)$ and all $\eps>0$ there is an
$\tilde f\in\bigcup_{L\ge0}\tilde\F_L$ such that $\nrm{\tilde f-g}_\infty<\eps$.

In particular, for each $g_k$ there is a sequence
$(f_{k,\ell})_{\ell\in\N}\subset\bigcup_{L\ge0}\tilde\F_L$
such that $f_{k,\ell}(x,y)\longto_{\ell\to\infty}g_k(x,y)$
almost everywhere $[\mu]$.
Also, 
$a_n\ninf a$ implies $$\chr[a_n<1]\ninf\chr[a<1].$$
Applying 
Lebesgue's dominated convergence theorem again,
\beq
\P(g_k(X,Y)<1) &=&
\lim_{\ell\to\infty}
\P(f_{k,\ell}(X,Y)<1).
\eeq

It follows that
\beq
R^* =
\lim_{k\to\infty}
\lim_{\ell\to\infty}
\P(f_{k,\ell}(X,Y)<1),
\eeq
which proves the claim.

\subsection{Proof of Lemma \ref{lem:approx}}
Rescaling $f\in\F_{2}$ to $g=2f/\g$,
Eq. (\ref{eq:lim-eta-cutoff}) is equivalent to claiming the existence of
an
$n_0({\eps})\in\N$ such that for all 
$n\geq n_0(\eps)$,
\beq
\abs{\srisk_{1,\gap_n}(g) - \risk_{1}(g)} & \leq & \eps/4.
\eeq
Since ${\gap_n=\nd^{-1}}$ decays to zero with increasing $n$,
it follows that
$
\srisk_{1,\nd^{-1}}(g) \ninf
\risk_{1}(g)
$
pointwise,
and so by
Lebesgue's dominated convergence theorem, we have that
\beq
\lim_{n\to\infty}\srisk_{1,{\nd^{-1}}}(g) 
= \risk_{1}(g),
\eeq
proving the claim.

\section*{Background on metric measure spaces}
\label{ap:metric_space_basics}
Here we provide some general relevant background on metric measure spaces.
Our metric space $(\X,\rho)$ is
doubling,
but in this section finite diameter is not assumed.
We recall some standard definitions.
A topological space is {\em Hausdorff} if every two distinct points
have disjoint neighborhoods. It is a standard (and obvious) fact
that all metric spaces are Hausdorff.

A metric space $\X$ is {\em complete} if every Cauchy sequence
converges to a point in $\X$.
Every metric space may be completed
by (essentially) adjoining to it the limits of all of its Cauchy
sequences \cite[Exercise 3.24]{MR0385023};
moreover, the completion is unique up to isometry
\cite[Section 43, Exercise 10]{MR0464128}.
We
implicitly
assume
throughout the paper
that
$\X$
is
complete.
Closed subsets of complete metric spaces
are also complete metric spaces under the inherited metric.

A
topological
space $\X$ is {\em locally compact} if every point $x\in\X$
has a compact neighborhood. It is a standard and easy fact
that complete doubling spaces are locally compact.
Indeed, consider any $x\in\X$
and the open $r$-ball about $x$, $B_r(x):=\set{y\in\X:\rho(x,y)<r}$.
We must show that 
$\cl(B_r(x))$
--- the closure of $B_r(x)$
--- 
is compact. To this end, it suffices to show that
$\cl(B_r(x))$
is {\em totally bounded}
(that is, has a finite $\eps$-covering number for each $\eps>0$),
since in complete metric spaces, a set is compact iff it is closed and
totally bounded \cite[Theorem 45.1]{MR0464128}.
Total boundedness follows immediately from the doubling property.
The latter posits a
constant $k$
and some
$x_1,\ldots,x_k\in\X$ such that
$B_r(x)\subseteq\cup_{i=1}^k B_{r/2}(x_i)$.
Then certainly
$
\cl(B_r(x))
\subseteq
\cup_{i=1}^k B_{2r/3}(x_i).
$
We now apply the doubling property recursively to each of the
$B_{2r/3}(x_i)$, until the radius of the covering balls becomes smaller than $\eps$.

We now recall some standard facts from measure theory.
Any topology on $\X$ (and in particular, the one induced by the metric
$\rho$),
induces the Borel $\sigma$-algebra $\mathscr{B}$.
A Borel probability measure is a function $\mu:\mathscr{B}\to[0,1]$
that is countably additive and normalized by $\mu(\X)=1$.
The latter is {\em complete} if for all
$A\subseteq B\in\mathscr{B}$ for which $\mu(B)=0$, we also
have $\mu(A)=0$. Any Borel $\sigma$-algebra may be completed
by defining the measure of any subset of a measure-zero set to be zero
\cite[Theorem 1.36]{rudin87}.
We
implicitly
assume
throughout the paper
that
$(\X,\mathscr{A},\mu)$ is a complete measure space,
where $\mathscr{A}$ contains all of the Borel sets.

The measure $\mu$ is said to be {\em outer regular}
if it can be approximated from above by open sets:
For every $E\in\mathscr{A}$, we have
\beq
\mu(E)=\inf\set{\mu(V):E\subseteq V, V~\text{open}}.
\eeq
A corresponding {\em inner regularity}
corresponds to approximability from below by compact sets:
For every $E\in\mathscr{A}$,
\beq
\mu(E)=\sup\set{\mu(K):K\subseteq E, K~\text{compact}}.
\eeq
The measure $\mu$ is {\em regular} if it is both inner and outer regular.
Any probability measure defined on the Borel $\sigma$-algebra
of a metric space is regular \cite[Lemma 1.19]{Kallenberg02}.
(Dropping the ``metric'' or ``probability'' assumptions
opens the door to various
exotic pathologies
\cite[Chapter 7]{MR2267655}, \cite[Exercise 2.17]{rudin87}.)

Finally, we have the following technical result, adapted
from \cite[Theorem 3.14]{rudin87}
to our setting:
\begin{theorem}
\label{thm:cc-dense}
Let $\X$ be a complete doubling metric space equipped
with a complete
probability measure $\mu$, such that all Borel sets are $\mu$-measurable.
Then $C_c(\X)$ (the collection of continuous functions with compact support)
is dense in $L_1(\mu)$.
\end{theorem}

\bibliographystyle{unsrtnat}

\end{document}